\documentclass[12pt, a4paper]{article}

\usepackage[a4paper]{geometry}
\usepackage{pgfplots}
\pgfplotsset{
    x tick style={color=black},
    y tick style={color=black}
}
\usepackage{authblk}

\usepackage{url}
\usepackage{listings}
\usepackage{amssymb}
\usepackage{graphicx}
\usepackage{algcompatible}
\usepackage{algorithm}
\usepackage{url}
\usepackage{rotating}
\usepackage{booktabs}
\usepackage{subfig}
\usepackage{amsmath}
\usepackage{bbm}

\renewcommand{\labelenumi}{(\alph{enumi})}
\renewcommand\theenumi\labelenumi

\usepackage[english]{babel}
\usepackage{lscape}

\usepackage[utf8]{inputenc}
\usepackage{xspace}
\usepackage{amsmath,amsthm,amssymb,mathtools}
\usepackage{lmodern}

\usepackage[algo2e,ruled,vlined,linesnumbered]{algorithm2e}

\usepackage{xcolor}
\usepackage{tikz}
\usepackage{graphicx}
\usepackage{booktabs} 
\usepackage{xspace}
\usepackage[algo2e,ruled,vlined,linesnumbered]{algorithm2e}
\usepackage{pdfpages}
\usepackage{multirow}
\usepackage{siunitx}
\usepackage{soul}
\usepackage{subfig}
\usepackage{algcompatible}
\usepackage{algorithm}
\usepackage{arydshln} 
\usepackage{ragged2e}

\renewcommand{\labelenumi}{\theenumi}
\renewcommand{\theenumi}{(\roman{enumi})}

\newtheorem{theorem}{Theorem}
\newtheorem{lemma}[theorem]{Lemma}

\newcommand{\oea}{\mbox{$(1 + 1)$~EA}\xspace}

\newcommand{\oplea}{\mbox{$(1+\lambda)$~EA}\xspace}
\newcommand{\mpoea}{\mbox{$(\mu+1)$~EA}\xspace}
\newcommand{\mplea}{\mbox{$(\mu+\lambda)$~EA}\xspace}
\newcommand{\mclea}{\mbox{${(\mu,\lambda)}$~EA}\xspace}
\newcommand{\oclea}{\mbox{$(1,\lambda)$~EA}\xspace}

\newcommand{\om}{\textsc{OneMax}\xspace}
\newcommand{\omt}{\textsc{OneMax$_{(0,1^n)}$}\xspace}

\newcommand{\cliff}{\textsc{Cliff}\xspace}

\newcommand{\jump}{\textsc{Jump}\xspace}

\DeclareMathOperator{\Poly}{Poly}

\newcommand{\R}{\ensuremath{\mathbb{R}}}

\newcommand{\N}{\ensuremath{\mathbb{N}}} 
\newcommand{\Z}{\ensuremath{\mathbb{Z}}}

\newcommand{\etalchar}[1]{$^{#1}$}

\begin{document}
{\sloppy
\title{When Non-Elitism Meets Time-Linkage Problems}

\author[a,b]{Weijie Zheng}
\author[a]{Qiaozhi Zhang}
\author[b]{Huanhuan Chen}
\author[a,c]{Xin Yao\thanks{Corresponding author}}
\affil[a]{Guangdong Provincial Key Laboratory of Brain-inspired Intelligent Computation, Department of Computer Science and Engineering, Southern University of Science and Technology, Shenzhen, China}
\affil[b]{School of Computer Science and Technology, University of Science and Technology of China, Hefei, China}
\affil[c]{CERCIA, School of Computer Science, University of Birmingham, Birmingham, United Kingdom}



\maketitle
\begin{abstract}
Many real-world applications have the time-linkage property, and the only theoretical analysis is recently given by Zheng, et al. (TEVC 2021) on their proposed time-linkage \om problem, \omt. However, only two elitist algorithms \oea and \mpoea are analyzed, and it is unknown whether the non-elitism mechanism could help to escape the local optima existed in \omt. In general, there are few theoretical results on the benefits of the non-elitism in evolutionary algorithms.
In this work, we analyze on the influence of the non-elitism via comparing the performance of the elitist \oplea and its non-elitist counterpart \oclea. We prove that with probability $1-o(1)$ \oplea will get stuck in the local optima and cannot find the global optimum, but with probability $1$, \oclea can reach the global optimum and its expected runtime is $O(n^{3+c}\log n)$ with $\lambda=c \log_{\frac{e}{e-1}} n$ for the constant $c\ge 1$. 
Noting that a smaller offspring size is helpful for escaping from the local optima, we further resort to the compact genetic algorithm where only two individuals are sampled to update the probabilistic model, and prove its expected runtime of $O(n^3\log n)$. Our computational experiments also verify the efficiency of the two non-elitist algorithms.
\end{abstract}

%
%



\maketitle

\section{Introduction}
The \emph{time-linkage} property, where the objective value depends on not only the current solution but also the historical ones, was first introduced to the evolutionary optimization community by Bosman~\cite{Bosman05}. Many real-world optimization problems have this property~\cite{Nguyen11}\footnote{Although the title of this literature is for the continuous optimization, its survey contains both continuous and discrete time-linkage real-world optimization problems.}, however, the theoretical analysis is quite few. Recently, Zheng, et al.~\cite{ZhengCY21} designed a time-linkage \om problem, \omt, which is based on the \om problem for the current time step but with an opposite preference of the first dimension of the immediate previous time step. The algorithms they discussed assume that the fitness is always computed based on the current stored individual, not the last evaluated individual. They proved that with probability of $1-o(1)$, the \oea will get stuck in one of two possible local optima, and cannot reach the global optimum. They also proved that introducing a not so small parent population size can significantly improve the probability of reaching the global optimum to $1-o(1)$, and condition on an event that happens with $1-o(1)$ probability, \mpoea can reach the global optimum in expected $O(n^2)$ fitness evaluations. 

Given the analysis in~\cite{ZhengCY21}, we could arrive at that the essential reason that \oea cannot leave the local optima once it reaches them stems from searching in the space $\{0,1\}^n$ but expecting reaching the optimum in the space of $\{0,1\}^{n+1}$ when encoding the solution in an $n$-dimensional bit string, storing the previous solution, and solving the \omt problem in the offline mode. In more details, putting the first bit position value space together with the current search space $\{0,1\}^n$ is indeed the $\{0,1\}^{n+1}$ space, and the aim is to find $(0(1^n))$. For the \omt problem, since the fitness of $\{0\}\times\{1\}\times\{0,1\}^{n-1}$ is always greater than the fitness of $\{1\}\times\{0,1\}^n$, when $(0(1*)), * \in\{0,1\}^{n-1}\setminus \{1^{n-1}\}$ is reached, any offspring generated from this point is in $\{1\}\times\{0,1\}^n$, and thus cannot enter into the population if the elitist selection is used, and the stagnation happens. Since $(1(1^n))$ is the local optimum for $\{1\}\times\{0,1\}^n$, if it is reached, any offspring generated from this point is in $\{1\}\times\{0,1\}^n$, and thus also cannot enter into the population if the elitist selection is used, and the stagnation happens. \cite{ZhengCY21} proved a $1-o(1)$ probability of getting stuck for the \oea, and introducing a not small parent population can reduce the stagnation probability to $o(1)$.

We note that some prediction on the future fitness might help to tackle the difficulty caused by the time-linkage property~\cite{Bosman05}, but similar to~\cite{ZhengCY21}, this work only focuses on how the basic evolutionary algorithms react to the time-linkage property without such advanced handling on the fitness. In this work, we first proved that introducing a non-trivial offspring population does not help to avoid reaching the local optima that could result in the stagnation for the elitist algorithms. More specifically, we proved that similar to \oea, \oplea with $1-o(1)$ probability will get stuck in one of the two possible local optima, see Theorem~\ref{thm:oplea}. Using the language we summarized above, we could give a more intuitive explanation here on why the introduction of the offspring population cannot help to reach the global optimum. The offspring population will result in a stronger selection pressure to make the first bit with value 1 than the \oea. When starting from the space of $\{1,0\}\times\{0*\}, *\in\{0,1\}^{n-1}$, the process with a high probability falls into $(0(1*)), * \in\{0,1\}^{n-1}\setminus \{1^{n-1}\}$  if the population size is not overly large, and when starting from $\{1\}\times\{1*\},*\in\{0,1\}^{n-1}$ the process might move to $\{1\}\times\{0*\}, *\in\{0,1\}^{n-1}$ when there are sufficient $0$s in the search point and eventually falls into $(0(1*)), * \in\{0,1\}^{n-1}\setminus \{1^{n-1}\}$  with a high probability, or might stay in $\{1\}\times\{1\}\times\{0,1\}^{n-1}$ and finally reach the $(1(1^n))$ with a high probability. 

As non-elitism is often regarded as a mechanism of escaping from local optima although with little theoretical evidence (see Section~\ref{subsec:nonelitism} for the literature review), we considered the situation for solving the time-linkage \omt. As a comparison to the \oplea, we pointed out that the \oclea can reach the global optimum with probability 1. It can be intuitively explained for the \oclea, and also for other non-elitist algorithms. Since the non-elitist algorithms can accept the inferior solutions to the best-so-far one, they could have a positive probability of moving from the current $\{a\}\times\{0,1\}^n, a \in \{0,1\}$ to any subspace $\{b\}\times\{0,1\}^n, b\in\{0,1\}$, and thus could reach the optimum of the space $\{0,1\}\times\{0,1\}^n$. We further proved that the expected runtime of the \oclea is $O(n^{3+c}\log n)$ when we take the offspring population size $\lambda=c\log_{\frac{e}{e-1}} n$ for any constant $c\ge 1$, see Theorem~\ref{thm:oclea}. Beyond the comma selection, noting that the less generated mutants, the larger probability to leave the local optimum once reached, we resorted to a more general non-elitist algorithm, the compact genetic algorithm (cGA) in which only two individuals are sampled to update the probabilistic model, and proved that it could reach the optimum of the \omt in expected $O(n^{2.5}\mu)$ runtime for the population size $\mu \in \Omega(\sqrt n \log n) \cap \Poly(n)$, see Theorem~\ref{thm:cga}. Our computational experiments also confirm the efficiency of the \oclea and the cGA, especially the superiority of the cGA, together with the non-convergence to the optimum of the \oplea.

The results that the non-elitism helps in the time-linkage problem is interesting and could be regarded a scenario for the advantage of the non-elitism as there is a long discussion in the evolutionary computation theory about the runtime benefit of the non-elitism, see the survey in~\cite{Doerr20} and more details in Section~\ref{subsec:nonelitism}. 

The reminder of this paper is organized as follows. The time-linkage \omt and the literature review for the non-elitist theory are introduced in Section~\ref{sec:pre}. Section~\ref{sec:oplea} discusses the performance of the elitist \oplea, and Sections~\ref{sec:oclea} and~\ref{sec:beyond} show the benefits of the non-elitist \oclea, and the cGA respectively. Our experimental verification is given in Section~\ref{sec:exp}, and Section~\ref{sec:con} concludes this work.

\section{Preliminaries}
\label{sec:pre}
\subsection{\omt}\label{subsec:omt}

\subsubsection{\omt}
The time-linkage problem is the category of the optimization problems in which the objective function relies not only on the current solution, but also the historical ones~\cite{Bosman05}. In~\cite{ZhengCY21}, the general time-linkage pseudo-Boolean problem $F:\{0,1\}^n$  $\times\dots\times\{0,1\}^n \rightarrow \R$ is defined by
\begin{align}
F(x^{ {t_0}},\dots,x^{ {t_0+\ell}})=\sum_{t=0}^{\ell} F_t(x^{ {t_0+t}};x^{ {t_0}},\dots,x^{ {t_0+t-1}})
\label{eq:F}
\end{align}
for consecutive $x^{ {t_0}},x^{ {t_0+1}},\dots,x^{ {t_0+\ell}}$ where $\ell \in \N$ and could be infinite, and $x^{ {t_0+t}}$ in $F_t$ is separated by a semicolon to indicate that it is the solution for the current time and others are the time-linkage historical solutions. To simplify the analysis and to aggressively (with overwhelming opposite weight for the immediate previous solution) show the difficulty that the time-linkage property could bring, Zheng, et al.~\cite{ZhengCY21} proposed the time-linkage version of the well-analyzed \om problem, the \omt problem, which specifies (\ref{eq:F}) by $\ell=1$, $F_0(x^{ {t_0}})=-nx_1^{ {t_0}}$, and $F_1(x^{ {t_0+1}};x^{ {t_0}})=\sum_{i=1}^n x_i^{ {t_0+1}}$. The formal maximization of the \omt problem $f:\{0,1\}\times\{0,1\}^n \rightarrow \Z$ is defined by
\begin{align}
f(x^{ {t-1}},x^{ {t}})=\sum_{i=1}^n x_i^{ {t}}-nx_1^{ {t-1}}
\label{eq:oms}
\end{align}
for two consecutive $x^{ {t-1}}=(x_1^{ {t-1}},\dots,x_n^{ {t-1}})$ and $x^{ {t}}=(x_1^{ {t}},\dots,x_n^{ {t}}) \in \{0,1\}^n$. Clearly, when we maximize the problem, the first bit position in the last time step prefers the $0$, and it is opposite to the preference of $1$ for the current first bit position, which might cause the possible difficulties for the optimization. Note that although~\cite{ZhengCY21} discussed solving the time-linkage problem in both offline and online modes, the online mode they discussed is just for the \oea and somehow reused the analysis for the offline mode, see more details in~\cite{ZhengCY21}. Hence, in the reminder of this paper, we just restrict ourselves to consider the $n$-bit string encoding for the offline mode, that is, we encode the $n$-bit string for the current solution that could be evolved by the variation (mutation) operator and store the solution of the immediate previous time for the time-linkage fitness evaluation. The global optimum for the \omt problem is the current solution being $(1^n)$ condition on that the first bit value of the immediate previous time is $0$. Since for such encoding in the offline mode we consider, the \emph{time} in the time-linkage problem is identical to the \emph{generation} for the evolutionary algorithms, we will not distinguish them in this paper.

\subsubsection{Local Optima Issue}
For the \omt problem, we only encode the current solution in an $n$-bit string, and expect to find $(1^n)$ at the current time step and require the first bit value $0$ of the immediate previous time, that is, we expect an optimum in $\{0,1\}^{n+1}$ space but the search process only happens in the subspace $\{0,1\}^n$. To help our understanding, we show a small example for the \oea on the \omt problem. Let's put the first bit value of the immediate previous time together with the current solution, that is, we say the whole space is $\{0,1\}^{n+1}$ where the first dimension is for the first bit value of the immediate previous time, and the optimum is $(0(1^n))$ where the pair of the brackets in $(1^n)$ is just a notation to indicate these dimensions are for the current solution. Assume in a certain step of the \oea, we start from a search point $X=(10^{n-1})$, the search (variation) only happens in the $n$-dimensional subspace of $\{0,1\}^{n+1}$, that is $\{1\} \times \{0,1\}^n$. 

As shown in~\cite{ZhengCY21}, the \oea will encounter the following two stagnation cases when the algorithm gets stuck and cannot escape.
\begin{lemma}[Lemma 3 in~\cite{ZhengCY21}]
Let $X^0,X^1,\cdots$ be the solution sequence generated by the \oea on the \omt problem. Denote 
\begin{itemize}
\item \emph{Event \textrm{I}}: There is a $g_0 \in \N$ such that $(X_{1}^{g_0-1},X_{1}^{g_0})=(0,1)$ and $X_{[2..n]}^{g_0} \neq 1^{n-1}$, 
\item \emph{Event \textrm{II}}: There is a $g_0 \in \N$ such that $(X_{1}^{g_0-1},X^{g_0})=(1,1^n)$. 
\end{itemize}
Then if at a certain generation among the solution sequence, \emph{Event \textrm{I}} or \emph{Event \textrm{II}} happens, then \oea cannot find the optimum of OneMax$_{(0,1^n)}$ in an arbitrary long runtime afterwards.
\label{lem:stuck}
\end{lemma}

Here we give a straightforward explanation of the local optima with the above language of searching in $\{0,1\}^n$ with expectation of finding the optimum in $\{0,1\}^{n+1}$. \emph{Event \textrm{I}} corresponds to $(0(1*))$ where $*\in\{0,1\}^{n-1}$ and $*\ne 1^{n-1}$. Then the search space is $\{1\}\times\{0,1\}^n$ as any variation starting from $X^{g_0}$ will have $X_1^{g_0}=1$ as the stored first bit value of the immediate previous time for the \omt fitness evaluation. However, the optimal solution in $\{1\}\times\{0,1\}^n$ is $(1(1^n))$ and the corresponding \omt value is $0$, which is less than the \omt value of at least $1$ for $(0(1*))$. Thus $(0(1*))$ is a local optimum.

\emph{Event \textrm{II}} corresponds to $(1(1^n))$. It is a local optimum since the search space is $\{1\}\times\{0,1\}^n$ and itself is the optimum of such a subspace.

With the above language, we could describe the results shown in~\cite{ZhengCY21}. The \oea will get stuck in one of the two local optima with probability of $1-o(1)$ w.r.t. the problem size $n$. For the \mpoea, one the one hand, the $\{1\}\times\{0,1\}^n$ search space will diminish due to its poor fitness value. On the other hand, the not small parent population size can avoid the local optima $(0(1*))$ taking over the whole population with $1-o(1)$ probability, thus it can find the global optimum with a high probability.

For ease of discussion,  we will call $(0(1*))$ and $(1(1^n))$ the two local optima corresponding to \emph{Event \textrm{I}} and \emph{Event \textrm{II}}, even when they are not actual local optima for some of the algorithms discussed in this paper.

\subsection{Literature Review on Non-Elitist Theory}\label{subsec:nonelitism}
In comparison to the prosperous theoretical work on the elitist evolutionary algorithms, the theory for the non-elitist evolutionary algorithms has been few. Recently, in his work about the classic non-elitist \mclea on the \jump function, Doerr~\cite{Doerr20} conducted a thorough survey on the theory for the classic non-elitist evolutionary algorithms, such as \oclea and \mclea. He divided the non-elitist theoretical work into three categories.
\begin{itemize}
\item When the selection pressure is low, the classic non-elitist evolutionary algorithms need exponential runtime. See the literatures mentioned in~\cite{Doerr20}. This category points the negative evidence of the classic non-elitist evolutionary algorithms.
\item When the selection pressure is high, the classic non-elitist evolutionary algorithm behaves similarly to its elitist counterpart and is essentially a pseudo-elitist algorithm. See the literatures mentioned in~\cite{Doerr20}. This category suggests the similar performance of the classic non-elitist evolutionary algorithms to their elitist counterparts.
\item Only two works~\cite{GarnierKS99,JagerskupperS07} show the examples for the benefit of the non-elitism. Garnier, et al.~\cite{GarnierKS99} proved that for a deceptive function the runtime for the $(1,1)$~EA is $O(2^n)$ while the elitist algorithms typically require $n^{\Theta(n)}$. As pointed in~\cite{Doerr20}, this example from~\cite{GarnierKS99} is an extreme case since $(1,1)$~EA actually performs a random walk and has no selection. A truly non-trivial example is given by Jagerskupper and Storch in~\cite{JagerskupperS07}. They proved that for the \cliff function with length $\frac n3$, the optimization time of the \oclea with $\lambda \ge 5\ln n$ is $e^{5\lambda}$ while it is at least $n^{n/4}$ for the \oplea with any offspring size $\lambda$.
\end{itemize}
From the above literature review, we could easily see that there is few work~\cite{GarnierKS99,JagerskupperS07} supporting the strength of the non-elitism mechanism. Our result of $O(n^{3+c}\log n)$ for the expected runtime of the \oclea with $\lambda=c\log_{\frac{e}{e-1}}n$ for constant $c\ge 1$ on the \omt (Theorem~\ref{thm:oclea}), comparing with that its elitlist counterpart \oplea cannot reach the optimum with probability of $1-o(1)$ (Theorem~\ref{thm:oplea}), will give one such positive evidence for the possible benefits of the classic non-elitist evolutionary algorithms.

Besides the classic non-elitist evolutionary algorithms, there are some literatures about other non-elitist randomized search heuristics, such as Metropolis algorithm, strong-selection-weak-mutation, and artificial immune systems, see~\cite{JansenW07,LissovoiOW19,PaixaoHST17,OlivetoPHST18,CorusOY18,CorusOY20}, but since our main contribution of this work is based on the classic evolutionary algorithm, we will not discuss them in details here. 

For one exception, we would like to briefly list some positive results for the estimation-of-distribution algorithms (EDAs)~\cite{Droste06,ChenTCY10,KrejcaW20}, which are more general non-elitist algorithms~\cite{Doerr20}, on the difficult problems, as we will discuss the compact genetic algorithm (cGA), one kind of EDAs, on the \omt problem in Section~\ref{sec:beyond}. For the \om with additive centered Gaussian noise with variance $\sigma^2$, Friedrich, et al.~\cite{FriedrichKKS17} proved a runtime of $O(\mu\sigma^2\sqrt n\log \mu n)$ with high probability for the cGA with population size $\mu =\omega(\sigma^2 \sqrt n \log n)$, while the simple hillclimber with a high probability cannot find the optimum in polynomial time for $\sigma^2>2$ (a similar result holds for the \oea~\cite{GiessenK16}), and with a high probability the \mpoea with $\mu \in \omega(1)\cap\Poly(n)$ cannot find the optimum in polynomial time for $\sigma^2\ge(na)^2$ for some $a\in\omega(1)$. For the well-analyzed multi-modal \jump function with jump size $k= o(n)$, Hasen{\"o}hrl and Sutton~\cite{HasenohrlS18} proved the runtime of $O(\mu n^{1.5} \ln n)$ with a high probability for the cGA with $\mu = \Omega(ne^{4k}+n^{3.5+\epsilon})$ for any small positive $\epsilon$. An improved runtime of $O(n\log n)$ with high probability for $\mu=\sqrt n \ln n$ when $k\le \frac{1}{20} \ln n - 1$ was given in~\cite{Doerr20algo}, while the lower bound of $n^k$ is held for the \oea~\cite{DrosteJW02} and the \mplea and even \mclea~\cite{Doerr20}. For the recently proposed DLB~\cite{LehreN19} function, Doerr and Krejca~\cite{DoerrK20} proved a runtime of $O(n^2\ln n)$ with high probability for the univariate marginal distribution algorithm, one kind of EDAs, with $\mu=\Theta(n\ln n)$, while the classic evolutionary algorithms need $O(n^3)$ in Lehre and Nguyen's work~\cite{LehreN19}.
 
\section{\oplea Cannot Find the Global Optimum}
\label{sec:oplea}
\subsection{\oplea}
The original \oplea was originally designed for the problems without time-linkage property. As discussed in Section~\ref{subsec:omt}, in this paper we consider the $n$-bit string encoding for the offline mode. Similar to the modification to the \oea and \mpoea in~\cite{ZhengCY21}, in order to tackle the time-linkage \omt problem, we encode an $n$-bit string for the current solution, and store the immediate previous solution for the fitness evaluation. The other procedure is identical to the original \oplea. The pseudo-code is shown in Algorithm~\ref{alg:oplEA}. The global optimum in this case is $X^{g^*}=(1^n)$ condition on $X^{g^*-1}_1=0$ for a certain generation $g^*$.
\begin{algorithm}[!ht]
    \caption{\oplea to maximize fitness function $f$ requiring two consecutive time steps}
    {\small
    \begin{algorithmic}[1]
    \STATE {Generate the random initial two generations $X^0=(X_{1}^0,\dots,X_{n}^0)$ and $X^1=(X_{1}^1,\dots,X_{n}^1)$}
    \FOR {$g=1,2,\dots$}
    \STATEx {\quad$\%\%$ \textit{Mutation}}
    \STATE {Independently generate $\tilde{X}^{(1)g}, \dots, \tilde{X}^{(\lambda)g}$, each via independently flipping each bit value of $X^g$ with probability $1/n$}
    \STATEx {\quad$\%\%$ \textit{Selection}}
    \STATE Let $S=\{\tilde{X}^{(i)g} \mid i\in[1..\lambda], \forall j\in[1..\lambda], f(X^{g},\tilde{X}^{(i)g}) \ge f(X^{g},\tilde{X}^{(j)g})\}$, and from $S$ uniformly at random select one element, denoted as $\tilde{X}^g$
    \IF {$f(X^g,\tilde{X}^g) \ge f(X^{g-1},X^g)$}
    \STATE $X^{g+1} = \tilde{X}^g$
    \ELSE
    \STATE $X^{g+1} = X^g$ and $X^{g}=X^{g-1}$.
    \ENDIF
    \ENDFOR
    \end{algorithmic}
    \label{alg:oplEA}
    }
\end{algorithm}

\subsection{Convergence Analysis}
In this subsection, we will analyze the convergence of the \oplea. As discussed in Section~\ref{subsec:omt}, there are two kinds of local optima, one is $(0(1*)), *\in \{0,1\}^{n-1}\setminus \{1^{n-1}\}$ that the first bit value of the immediate previous time becomes $0$ and the first bit value of the current time is $1$ before the optimum is found, the other is $(1(1^n))$ that the first bit value of the immediate previous time becomes $1$ and the current solution is $1^n$. With the same reason discussed in Section~\ref{subsec:omt}, it is not difficult to see that once the \oplea reaches one of the two, the algorithm will get stuck and cannot find the optimum further. In the following, we will show that with high probability, one of the two kinds of local optima will be reached and the algorithm cannot find the optimum.

Firstly, the following lemma estimates the probability of $|\tilde{X}^g| -  |X^g| < \ln \lambda$ condition on the event that $|\tilde{X}^g| > |X^g|$.
\begin{lemma}
Let $\lambda \ge e^e$. Suppose $\tilde{X}^g, g>0$ is generated from $X^g$ via Steps 3-4 in Algorithm~\ref{alg:oplEA}. Let $a$ be the number of zeros in $X^g$. Then we have 
\begin{align*}
\Pr[|\tilde{X}^g| - |X^g|<{\lceil \ln\lambda \rceil} \mid |\tilde{X}^g| > |X^g|] \ge 1-\frac{\lambda(ne+a\lambda)a^{\ln\lambda-1}}{(n\ln\lambda)^{\ln\lambda}}.
\end{align*}
\label{lem:ln}
\end{lemma}
\begin{proof}
We calculate
\begin{align*}
\Pr[|\tilde{X}^g|& - |X^g|\ge{\lceil \ln\lambda \rceil}] \le {1-\left(1-\binom{a}{\lceil \ln\lambda \rceil}\left(\frac{1}{n}\right)^{\lceil \ln\lambda \rceil}\right)^{\lambda}}
\le {\lambda\binom{a}{\lceil \ln\lambda \rceil}\left(\frac{1}{n}\right)^{\lceil \ln\lambda \rceil}}\\
&\le \lambda\left(\frac{ea}{n\lceil \ln\lambda \rceil}\right)^{\lceil \ln\lambda \rceil} \le  \lambda\left(\frac{ea}{n\ln\lambda}\right)^{\ln\lambda}
=\left(\frac{e^2a}{n\ln\lambda}\right)^{\ln\lambda}
\end{align*}
where the last inequality uses $n\lceil \ln\lambda \rceil \ge en \ge ea$ from $\lambda \ge e^e$,
and
\begin{align}
\Pr[|\tilde{X}^g|& > |X^g|]\ge1-\left(1-a\frac 1n \left(1-\frac 1n\right)^{n-1}\right)^{\lambda}
\ge 1-\left(1-\frac{a}{ne}\right)^{\lambda}\\
&\ge1-\frac{1}{1+\frac{a\lambda}{ne}}=\frac{a\lambda}{ne+a\lambda}.
\label{eq:improve}
\end{align}
Then we have
\begin{align*}
\Pr[|\tilde{X}^g|& - |X^g|\ge {\lceil \ln\lambda \rceil}\mid |\tilde{X}^g| > |X^g|]=\frac{\Pr[|\tilde{X}^g| - |X^g|\ge{\ln\lambda}]}{\Pr[|\tilde{X}^g| > |X^g|]}\\
&\le\frac{\left(\frac{e^2a}{n\ln\lambda}\right)^{\ln\lambda}}{\frac{a\lambda}{ne+a\lambda}}
=\frac{\lambda(ne+a\lambda)a^{\ln\lambda-1}}{(n\ln\lambda)^{\ln\lambda}}.
\qedhere
\end{align*}
\end{proof}

There are four possible initial situations for the pair of $X^{0}_1$ and $X^1_1$, $(0,1), (0,0), (1,0),$ and $(1,1)$. Due to the random initialization, with probability at least $1-\exp(-(n-1)/8)$, we know that $\sum_{i=2}^n X^1_i < \frac 34 n$ happens. Hence, if $(X^0_1,X^1_1)=(0,1)$, then $(0(1*)), *\in \{0,1\}^{n-1}\setminus \{1^{n-1}\}$ local optimum is already reached. 

If $(X^0_1,X^1_1)=(0,0)$, we just consider the subprocess on the generations in which a strict increase of the number of $1$s happens, and we could show that before the number of $0$s decreases from $[n^c,n^c+\ln \lambda]$ for some constant $c<0.5$ to $[\ln \lambda, 2\ln \lambda]$, with probability $1-o(1)$, the first bit value will change from $0$ to $1$. See Lemma~\ref{lem:00}.
\begin{lemma}
Consider using \oplea to optimize the $n$-dimensional \omt problem. Given any $c <0.5$, let $\ln \lambda \le n^c$ and $\lambda \ge e^e$. If $(X^0_1,X^1_1)=(0,0)$, then with probability at least $1-\frac{2\ln \lambda}{n^c} - 3n\left(\frac{2e^2}{n^{1-c}\ln \lambda}\right)^{\ln \lambda}$, $(0(1*))$ local optima will be reached at one certain generation starting from such initialization.
\label{lem:00}
\end{lemma}
\begin{proof}
Starting from $(X^0_1,X^1_1)=(0,0)$, we could assume that $X_1^2,\dots,X_1^{g_0}$ all take the value of $0$ when the number of $0$s of the $\{2,\dots,n\}$ bit positions of $X^{g_0}$, denoted as $a$, is in $[n^c,n^c+\ln\lambda]$ for any given constant $c < 0.5$. Otherwise, $(0(1*))$ local optima is already reached at some generation $g'<g_0$. We start from such situation that $a\in [n^c,n^c+\ln\lambda]$.

Under the condition that the number of 1s in the individual increases by at least $1$ in one generation, we suppose that the number of bits changing from 0 to 1 in this generation is $m$, and the number of $1$s increases by $m'$. Then $1\le m' \le m \le a$ and the probability that the first bit contributes one 0 is  
\begin{align*}
\frac{\binom{a-1}{m-1}}{\binom{a}{m}}=\frac{m}{a}\geq\frac{m'}{a}.
\end{align*} 
For the subprocess that the number of 1s in the individual increases by at least $1$ in one generation, we consider the number of 0s decreases from $[n^c,n^c+\ln\lambda]$ to $[\ln\lambda,2\ln\lambda]$. Let $a_1,\dots,a_k$ for some $k\in \N$ be such sequence of the number of $0$s. 
With Lemma~\ref{lem:ln}, we know the probability of the first bit value changing from $0$ to $1$ before the rest bit values all become $1$ is at least  
\begin{align*}
\bigg(\prod_{i=1}^{k-1}&\left(1-\frac{\lambda(ne+a_i\lambda)a_i^{\ln\lambda-1}}{(n\ln\lambda)^{\ln\lambda}}\right)\bigg)\left(1-\prod_{i=1}^{k-1}\left(1-\frac{a_i-a_{i+1}}{a_i}\right)\right)\\
=&\bigg(\prod_{i=1}^{k-1}\left(1-\frac{\lambda(ne+a_i\lambda)a_i^{\ln\lambda-1}}{(n\ln\lambda)^{\ln\lambda}}\right)\bigg)\left(1-\frac{a_k}{a_1}\right)\\
\ge & \left(1-\frac{\lambda(ne+a_1\lambda)a_1^{\ln\lambda-1}}{(n\ln\lambda)^{\ln\lambda}}\right)^{k-1}\left(1-\frac{2\ln \lambda}{n^c}\right)\\
\ge & \left(1-\frac{\lambda(ne+2n^c\lambda)(2n^c)^{\ln\lambda-1}}{(n\ln\lambda)^{\ln\lambda}}\right)^{2n^c}\left(1-\frac{2\ln \lambda}{n^c}\right)\\
\ge & 1-\frac{2\ln \lambda}{n^c} - \frac{2n^c\lambda(ne+2n^c\lambda)(2n^c)^{\ln\lambda-1}}{(n\ln\lambda)^{\ln\lambda}}\\
\ge & 1-\frac{2\ln \lambda}{n^c} - 3n\left(\frac{2e^2}{n^{1-c}\ln \lambda}\right)^{\ln \lambda},
\end{align*}
where the first inequality uses $\ln \lambda \le n^c$ and the last inequality uses $ne+2n^c\lambda \le 3n\lambda$ from $\lambda \ge e^e$.
\end{proof}

If $(X^0_1,X^1_1)=(1,0)$, then any possible generated $\tilde{X}^1$ can enter in the next generation, that is, $X^2=\tilde{X}^1$. Since with probability of $1-o(1)$, the number of $0$s will be greater than $n^c+\ln \lambda$ and thus $X^2\ne 1^n$, the process turns to the $(0,0)$ case discussed in Lemma~\ref{lem:00}. Hence, we have the following lemma.
\begin{lemma}
Given any $c <0.5$, let $n\in\N$ with $n^{1-c} \ge 16$, $\ln \lambda \le n^c$ and $\lambda \ge e^e$. Consider using \oplea to optimize the $n$-dimensional \omt problem.  If $(X^0_1,X^1_1)=(1,0)$ and $\sum_{i=2}^n X^1_i < \frac 34 n$, then with probability at least $1-2n^c\lambda \left(\frac{2e}{n}\right)^{n/8}-\frac{2\ln \lambda}{n^c} - 3n\left(\frac{2e^2}{n^{1-c}\ln \lambda}\right)^{\ln \lambda}$, $(0(1*))$ local optima will be reached at one certain generation starting from such initialization.
\label{lem:10}
\end{lemma}
\begin{proof}
Since $(X^0_1,X^1_1)=(1,0)$, we know that any generated $\tilde{X^1}$ will have the fitness value $f(X^1,\tilde{X^1}) \ge 0 > f(X^0,{X^1})$, and thus surely becomes $X^{2}$. Since $\sum_{i=2}^n X^1_i < \frac 34 n$, denoting $a$ as the number of $0$s in $X^1_{[2..n]}$ we know that 
\begin{align*}
\Pr&\left[|X_{[2..n]}^2| < n-1-a-n^c-\ln n\right] \ge \Pr\left[|X_{[2..n]}^2| < n-1-a-2n^c\right] \\
\ge & \left(1-\sum_{k=a-2n^c+1}^{a} \binom{a}{k} \frac1{n^k}\right)^{\lambda} \ge \left(1-\sum_{k=a-2n^c+1}^{a} \left(\frac{ea}{nk}\right)^k\right)^{\lambda} \\
\ge & \left(1-2n^c \left(\frac{ea}{n(a-2n^c+1)}\right)^{a-2n^c+1}\right)^{\lambda}  \ge\left(1-2n^c \left(\frac{2e}{n}\right)^{\frac12 (a+1)}\right)^{\lambda}\\
\ge & 1-2n^c\lambda \left(\frac{2e}{n}\right)^{\frac12 (a+1)} \ge 1-2n^c\lambda \left(\frac{2e}{n}\right)^{n/8},
\end{align*} 
where the antepenultimate inequality uses $2n^c \le \frac12 \cdot\frac 14 n \le \frac12 (a+1)$ for $n^{1-c} \ge 16$, and the last inequality uses $a+1\ge \frac 14n$ from $\sum_{i=2}^n X^1_i < \frac 34 n$. 
Condition on that there are at least $n^c+\ln n$ zeros in $X^{2}_{[2..n]}$, if $X^{2}_1=1$, then $(0(1*))$ local optima is already reached. Otherwise the following process is identical to the one discussed in Lemma~\ref{lem:00}. Hence, the probability that $(0(1*))$ local optima will be reached at one certain generation starting from such initialization is at least
\begin{align*}
\bigg( 1-&2n^c\lambda \left(\frac{2e}{n}\right)^{n/8} \bigg)\\
&\cdot \left(\Pr[X^{2}_1=1]+(1-\Pr[X^{2}_1=1])\left(1-\frac{2\ln \lambda}{n^c} - 3n\left(\frac{2e^2}{n^{1-c}\ln \lambda}\right)^{\ln \lambda}\right)\right)\\
\ge&  1-2n^c\lambda \left(\frac{2e}{n}\right)^{n/8}-\frac{2\ln \lambda}{n^c} - 3n\left(\frac{2e^2}{n^{1-c}\ln \lambda}\right)^{\ln \lambda}.
\qedhere
\end{align*}
\end{proof}

If $(X^0_1,X^1_1)=(1,1)$, then if the first bit value changes from $1$ to $0$ in some generation before the number of $0$s decreases to $n^c$, the further process turns to the $(1,0)$ case and then $(0,0)$ case, and thus with probability of $1-o(1)$, $(0(1*))$ local optima will be reached in some future generation. Otherwise, if the first bit value stays at $1$ when the number of $0$s decreases to $n^c$, then with high probability, the first bit stays at $1$ when the rest $n-1$ bits all have the value of $1$. See Lemma~\ref{lem:11}.
\begin{lemma}
Given any $c <0.5$, let $n\in\N$ with $n^{1-c} \ge 16$, $\ln \lambda \le n^c$ and $\lambda \ge e^e$.  Consider using \oplea to optimize the $n$-dimensional \omt problem. If $(X^0_1,X^1_1)=(1,1)$, then with probability at least $1-2n^c\lambda \left(\frac{2e}{n}\right)^{n/8}-\frac{2\ln \lambda}{n^c} - 3n\left(\frac{2e^2}{n^{1-c}\ln \lambda}\right)^{\ln \lambda} -\frac{\lambda}{n^{1-2c}}-(n-1)e^{-\frac{n^c}{e}}$, $(0(1*))$ or $(1(1^n))$ local optima will be reached at one certain generation starting from such initialization.
\label{lem:11}
\end{lemma}
\begin{proof}
Starting from $(X^0_1,X^1_1)=(1,1)$, we could assume that $X_1^2,\dots,X_1^{g_0}$ all take the value of $1$ when the number of $0$s of the $\{2,\dots,n\}$ bit positions of $X^{g_0}$, denoted as $a$, is not greater than $n^c$ for any given constant $c < 0.5$. Otherwise, $(X_1^{g'-1},X_1^{g'})=(1,0)$ will happen at some generation $g'<g_0$, and we just turn to the case discussed in Lemma~\ref{lem:10}.

Condition on that the first bit value stays at $1$, let $\tilde{T}$ be the time that the rest $n-1$ bit values all become $1$. Suppose that the process is in the $g$-th generation. Let $Y$ be a random variable that is generated from $X^g$ via the standard bit-wise mutation, and any subset $S\in[1..n]$. It is not difficult to see that $|\tilde{X}_{S}|$ stochastically dominates\footnote{See more about stochastic dominance in~\cite{Doerr19tcs}.} $|Y_{S}|$. Then for a certain bit position $i_0\in[2..n]$, let $S=[2..n]\setminus\{i_0\}$, and we have 
\begin{align*}
\Pr[|\tilde{X}^g_{S}| \ge |X^g_{S}|] \ge \Pr[|Y_{S}| \ge |X^g_{S}|] \ge \left(1-\tfrac{1}{n}\right)^{n-2}.
\end{align*}
In this case, the bit position $i_0$ is neutral and the probability of being $1$ condition on $|\tilde{X}^g_{S}| \ge |X^g_{S}|$ is $\tfrac 1n$. Hence, the probability that the $i_0$-th bit value becomes $1$ is at least $\tfrac1n  \left(1-\tfrac{1}{n}\right)^{n-2}$. Then the event that the $i_0$-th bit value stays at $0$ in $t$ generations is at most
\begin{align*}
\left(1-\frac 1n\left(1-\frac1n\right)^{n-2}\right)^{t} \le \left(1-\frac{1}{en}\right)^t,
\end{align*}
and a union bound gives
\begin{align*}
\Pr[\tilde{T} > n^{1+c} \mid \text{the first bit value stays at } 1] \le (n-1)\left(1-\frac{1}{en}\right)^{n^{1+c}} \le (n-1)e^{-\frac{n^c}{e}}.
\end{align*}

Since the probability for $X^{g+1}=\tilde{X^g}$ with $\tilde{X}_1^g=0$ for any $g\ge g_0$ is at most $\lambda \tfrac 1n \tfrac an \le \frac{\lambda}{n^{2-c}}$ where $a$ is the number of $0$s in $X^g$, we know that the probability that $(1(1^n))$ local optimum is reached within $n^{1+c}$ generations is at least
\begin{align*}
\bigg(1-&\frac{\lambda}{n^{2-c}}\bigg)^{n^{1+c}}\left(1-(n-1)e^{-\frac{n^c}{e}}\right)
\ge 1-\frac{\lambda}{n^{1-2c}}-(n-1)e^{-\frac{n^c}{e}}.
\end{align*}

Hence, also considering the probability that $(0(1*))$ local optima is reached after $g<g_0$ with $(X_1^{g'-1},X_1^{g'})=(1,0)$, we have the probability to get stuck is at least $1-\frac{\lambda}{n^{n^{n/4}}}-\frac{2\ln \lambda}{n^c} - 3n\left(\frac{2e^2}{n^{1-c}\ln \lambda}\right)^{\ln \lambda} -\frac{\lambda}{n^{1-2c}}-(n-1)e^{-\frac{n^c}{e}}$.
\end{proof}

Now we establish the main result for the non-convergence of the \oplea on \omt with high probability.
\begin{theorem}\label{thm:oplea}
Let $n\ge 64, \lambda \ge e^e$ and $\ln \lambda \le n^{1/3}$. Then with probability at least $1-3n\left(\frac{2e^2}{n^{2/3}\ln \lambda}\right)^{\ln \lambda} -\frac{3\lambda}{n^{1/3}}-ne^{-\frac{n^{1/3}}{e}}$, the \oplea cannot find the optimum of the \omt problem.
\end{theorem}
\begin{proof}
It is not difficult to see that from the Chernoff inequality, we have $\Pr[\sum_{i=2}^n X^1_i < \frac 34 n] \ge 1-\exp(-(n-1)/8)$. We consider the process with this initialization.

If $(X^0_1,X^1_1)=(0,1)$, then the $(0(1*))$ local optima is reached already. Together with the initial cases $(X^0_1,X^1_1)=(0,0),(1,0)$, and $(1,1)$ from Lemmas~\ref{lem:00} to~\ref{lem:11} taking $c=1/3$, we have the probability that the $(0(1*))$ or $(1(1^n))$ local optima is reached during the optimization process is at least
\begin{align*}
\bigg(1-&e^{-{(n-1)}/{8}}\bigg)
 \left(1-2n^{1/3}\lambda \left(\frac{2e}{n}\right)^{n/8}-\frac{2\ln \lambda}{n^{1/3}} \right.\\
 & \left. - 3n\left(\frac{2e^2}{n^{2/3}\ln \lambda}\right)^{\ln \lambda} -\frac{\lambda}{n^{1/3}}-(n-1)e^{-{n^{1/3}}/{e}}\right)\\
\ge & 1-e^{-{(n-1)}/{8}}- 3n\left(\frac{2e^2}{n^{2/3}\ln \lambda}\right)^{\ln \lambda} -\frac{3\lambda}{n^{1/3}}-(n-1)e^{-n^{1/3}/{e}}\\
\ge & 1-3n\left(\frac{2e^2}{n^{2/3}\ln \lambda}\right)^{\ln \lambda} -\frac{3\lambda}{n^{1/3}}-ne^{-\frac{n^{1/3}}{e}},
\end{align*}
where the first inequality uses that $2n^{1/3} \left(\frac{2e}{n}\right)^{n/8} \le n^{-1/3}$ when $n\ge 64$ and $2\ln \lambda \le \lambda$ for $\lambda \ge e^e$, and the last inequality uses that ${(n-1)}/{8} \ge -n^{1/3}/{e}$ for $n\ge 64$.
\end{proof}
Note that in Theorem~\ref{thm:oplea}, we require that $\lambda$ should not be too large, like $\ln \lambda \le n^{1/3}$, or more strictly, we need $3\lambda \le n^{1/3}$ to make the probability greater than $0$. One reason for the restrictions on $\lambda$ stems from that we consider the probability of the first bit value changing from $1$ to $0$, and use the uniform upper bound $\lambda\frac1n\frac{a}{n} \le \frac{\lambda}{n^{2-c}}$ for all the number of $0$s $a\le n^c, c<0.5$, which is too loose for a large offspring size. When the offspring size $\lambda$ becomes large, the high selection pressure among the offspring will be more likely to reject the one with the first bit value of $0$ which contributes a $-1$ to the fitness. Hence, we conjecture that for the large $\lambda$, the \oplea cannot find the optimum of the \omt problem also with $1-o(1)$ probability. However, if the offspring population size becomes extremely large, then starting from the initial $X^1_1=0$, with a good chance we could generate $\tilde{X}^1=1^n$ and the optimum is reached. It might need more detailed discussion in the future. Such analyses might be connected to the work of the impact of population size, such as~\cite{ChenTCY12,Sudholt20}.

\section{Comma Selection Can Help}
\label{sec:oclea}
\subsection{\oclea}
As mentioned in Section~\ref{subsec:omt}, this paper only discusses solving the \omt problem in the offline mode with $n$-bit string encoding for the current search point and storing the previous solution for the fitness evaluation. The \oclea algorithm for the problem requiring two consecutive time steps is similar to the traditional \oclea algorithm, and the only difference is that the best solution of the last step is stored for the fitness evaluation. Without confusion, we still call the time-linkage version the \oclea algorithm. See Algorithm~\ref{alg:oclea} for details. The global optimum in this case is that $\tilde{X}^g=1^n$ condition on $X_1^g=0$. 

\begin{algorithm}[!ht]
    \caption{\oclea to maximize fitness function $f$ requiring two consecutive time steps}
    {\small
    \begin{algorithmic}[1]
    \STATE {Generate the random initial two generations $X^0=(X_{1}^0,\dots,X_{n}^0)$ and $X^1=(X_{1}^1,\dots,X_{n}^1)$}
    \FOR {$g=1,2,\dots$}
    \STATEx {\quad$\%\%$ \textit{Mutation}}
    \STATE {Independently generate $\tilde{X}^{(1)g}, \dots, \tilde{X}^{(\lambda)g}$, each via independently flipping each bit value of $X^g$ with probability $1/n$}
    \STATEx {\quad$\%\%$ \textit{Selection}}
    \STATE Let $S=\{\tilde{X}^{(i)g} \mid i\in[1..\lambda], \forall j\in[1..\lambda], f(X^{g},\tilde{X}^{(i)g}) \ge f(X^{g},\tilde{X}^{(j)g})\}$, and from $S$ uniformly at random select one element, denoted as $\tilde{X}^g$
    \STATE $X^{g+1}=\tilde{X}^g$
    \ENDFOR
    \end{algorithmic}
    \label{alg:oclea}
    }
\end{algorithm}

\subsection{Convergence}
Corresponding to the two local optima existed in the \oea, since the fitness value of $X^g$ together with its ancestor $X^{g-1}$ does not join in the selection, we know that 
\begin{itemize}
\item for $(0(1*))$ local optima, that is for a certain generation $g>0$, $(X^{g-1}_1,X^g_1)=(0,1)$ but $X^g\neq 1^n$, we know that $\tilde{X}^{g}$ will have a fitness $f(X^g_1, \tilde{X}^{g}) \le 0 < f(X^{g-1}_1,X^g_1)$. However, since $(X^{g-1}_1,X^g_1)$ does not join in the selection, see Step 5, $\tilde{X}^g$ can enter into the next generation, and no stagnation happens.
\item for $(1(1^n))$ local optimum, that is for a certain generation $g>0$, $(X^{g-1}_1,X^g)=(1,1^n)$, we know that $\tilde{X}^{g} \neq 1^n$ will have a fitness $f(X^g_1, \tilde{X}^{g}) < 0 = f(X^{g-1}_1,X^g_1)$. However, since $(X^{g-1}_1,X^g_1)$ does not join in the selection, see Step 5, $\tilde{X}^g$ can still enter into the next generation as in Step 5, and no stagnation happens.
\end{itemize}
Hence, we could know that with probability $1$, the optimum of the \omt can be reached. As a comparison with the probability of the \oplea stagnation in Theorem~\ref{thm:oplea}, we put this into the following theorem although it is trivial.
\begin{theorem}
Let $n\in \N$ and $X^{0},X^{1},\dots$ be the sequence of the solutions for the \oclea on the $n$-dimensional \omt problem. Then with probability $1$, there exists $g_0 \in \N$ such that $f(X^{g_0-1},X^{g_0})=n$.
\label{thm:prob1}
\end{theorem}
Note that the probability of $1$ for reaching the global optimum is not the unique property for the \oclea, but also a general character for all non-elitist algorithms in which the inferior solution to the best-so-far one can still have a positive probability to be accepted and enter into the next generation. Using the language we discussed in Section~\ref{subsec:omt}, it can be intuitively explained that since the non-elitist algorithms can accept the inferior solutions, they could have a positive probability of moving from the current $\{a\}\times\{0,1\}^n, a \in \{0,1\}$ to any subspace $\{b\}\times\{0,1\}^n, b\in\{0,1\}$, and thus could reach the optimum of the space $\{0,1\}\times\{0,1\}^n$ with probability of $1$. 

\subsection{Runtime Analysis}
Although the \oclea does not stagnate in the two local optima, it needs different iterations in expectation to leave them. The first case that $(X^{g-1}_1,X^g_1)=(0,1)$ but $X^g\neq 1^n$ will not be satisfied as the $(X^g_1,X^{g+1}_1)$ will be $(1,*)$ where $*$ could be $0$ or $1$ depending on the $\tilde{X}^g$. But for the second case, that is, the $(1(1^n))$ local optimum, $\tilde{X}^g_1=0$ is the only case to leave the stagnation. Noting that the probability of reaching $\tilde{X}_1^{g}=0$ is less than the probability of $X^{g+1}\ne X^g$, that is, all generated $\lambda$ offspring needs to change at least one bit value from $1$ to $0$, which happens with probability
\begin{equation}
\left(1-\left(1-\frac 1n\right)^n\right)^{\lambda} \le \left(1-\left(1-\frac 1n\right)\frac 1e\right)^{\lambda} \le \left(\frac{2e-1}{2e}\right)^{\lambda}
\label{eq:probupper}
\end{equation}
where we use $n\ge 2$ for the last inequality. Hence, once the second stagnation case happens, we need at least $\left(\frac{2e}{2e-1}\right)^{\lambda}$ expected iterations to leave the stagnation. It is not difficult to see that from the random initialization, for any $g>0$, $X_1^g$ stochastically dominates the random variable $Y$ that obeys the Bernoulli distribution with success probability of $\tfrac 12$. Hence, the $(1(1^n))$ case happens for the first time the current solution becomes $1^n$ with probability at least $\tfrac 12$, and then we could easily obtain the following lower bound.
\begin{theorem}
The expected number of fitness evaluations for the \oclea reaching the optimum of the \omt problem is $\Omega\left(\lambda\left(\frac{2e}{2e-1}\right)^{\lambda}\right)$.
\label{thm:lower}
\end{theorem}
By Theorem~\ref{thm:lower}, we could easily see that to efficiently (in polynomial time) solve the \omt problem we need the offspring size $\lambda =O(\log n)$.

It is not difficult to see that before the $(1(1^n))$ case happens, the process of the \oclea optimizing the \omt problem is identical to the process of the classic (non-time-linkage) \oclea optimizing the \om problem. Hence, here we take the existing result for the classic \oclea on the \om problem from~\cite{RoweS14}.
\begin{theorem}[\cite{RoweS14}] Consider using the (non-time-linkage) \oclea to solve some $n$-dimensional problem,
\begin{itemize}
\item[(a)] If $\lambda \ge \log_{\tfrac{e}{e-1}}n$, then the expected number of function evaluations for the \oclea solving the \om problem is $O(n\log n+n\lambda)$;
\item[(b)] If $\lambda \le (1-\epsilon)\log_{\tfrac{e}{e-1}}n$ for some $\epsilon \in (0,1]$, the for any function with a unique global optimum, with probability $1-2^{-\Omega(n^{\epsilon/2})}$, the \oclea needs at least $2^{cn^{\epsilon/2}}$ running time for some constant $c>0$.
\end{itemize}
\label{thm:om}
\end{theorem}

Now we consider that the $(1(1^n))$ case happens for the first time at generation $g$. We could pessimistically let the first offspring $\tilde{X}^{(1)g}$ flip only the first bit, the other $\lambda-1$ offspring change by at least one bit, and $\tilde{X}^{g}=\tilde{X}^{(1)g}$, which happens with the probability at least
\begin{equation}
\frac{1}{\lambda} \frac 1n \left(1-\frac 1n\right)^{n-1}\left(1-\left(1-\frac 1n\right)^{n}\right)^{\lambda -1},
\label{eq:problower}
\end{equation}
where $1/\lambda$ stems from the lower bound on the probability to select the first individual, if all others have only one bit flipped.
Hence, condition on the event that $X^{g+1}\ne X^g$, which has the probability of $(1-(1-\frac1n)^n)^{\lambda}$, with probability at least 
\begin{align*}
\frac{1}{\lambda} \frac 1n \left(1-\frac 1n\right)^{n-1}&\left(1-\left(1-\frac 1n\right)^{n}\right)^{\lambda -1}\cdot \frac{1}{(1-(1-\frac1n)^n)^{\lambda}} \\
\ge & \frac{1}{en\lambda (1-(1-\frac1n)^n)},
\end{align*}
$\tilde{X}^g=(01^{n-1})$ happens, otherwise we move to the process identical to the classic \oclea solving the \om problem, which requires $O(n\log n/\lambda+n)$ generations to back to the second stagnation case or the optimum is reached via Theorem~\ref{thm:om}~$(a)$ for $\lambda\ge \log_{\tfrac{e}{e-1}}n$. Using the Wald’s equation, we know that the expected generations to reach $\tilde{X}^g=(01^{n-1})$ for a certain $g$ or that the optimum is reached is 
\[
O\left(en\lambda \left(1-\left(1-\frac1n\right)^n\right)\left(\frac{n\log n}{\lambda}+n\right)\right)=O\left(en^2\lambda\left(1-\left(1-\frac1n\right)^n\right)\right).
\]
Together with the expected generations that $X^{g+1}\ne X^g$ happens, we know that we need
\begin{align*}
O\left(\frac{en^2\lambda}{ \left(1-\left(1-\frac1n\right)^n\right)^{\lambda-1}}\right)&=O\left(en^2\lambda \left(\frac{e}{e-1}\right)^{\lambda-1}\right)\\
&=O\left((e-1)n^2\lambda \left(\frac{e}{e-1}\right)^{\lambda}\right)
\end{align*}
expected generations to reach $\tilde{X}^g=(01^{n-1})$ for a certain $g$ or that the optimum is reached. 
Once $X^{g+1}=\tilde{X}^g=(01^{n-1})$ is obtained before the optimum is reached, we know that with probability at least
\begin{align*}
1-\bigg(1-\frac 1n &\left(1-\frac 1n\right)^{n-1}\bigg)^{\lambda}
\ge 1-\left(1-\frac{1}{ne}\right)^{\lambda}\\
&\ge 1-\frac{1}{1+\frac{\lambda}{ne}}=\frac{\lambda}{ne+\lambda},
\end{align*}
we have $\tilde{X}^{g+1}=1^n$, that is, the optimum is found. Again with Wald's equation, we know the expected generations for the \oclea to find the optimum is $O\left((e-1)\frac{\lambda+ne}{\lambda}n^2\log n \left(\frac{e}{e-1}\right)^{\lambda}\right)$. Since in each generation, $\lambda$ offspring are evaluated, together with our requirement of $\lambda=O(\log n)$ for a polynomial runtime by Theorem~\ref{thm:lower}, we have the following runtime result.
\begin{theorem}
Let $n\in N$. The expected number of fitness evaluations for the \oclea with $\lambda = c\log_{\tfrac{e}{e-1}}n$ for any constant $c\ge 1$ on the \omt problem is $O(n^{3+c}\log n)$.
\label{thm:oclea}
\end{theorem}
From the above analyses, we could see the efficient range of the offspring size is $\lambda = c\log_{\tfrac{e}{e-1}}n$ for any constant $c\ge 1$, since $\lambda=\omega(\log n)$ will make the probability of escaping the $(1(1^n))$ local optimum very small and will result in a super-exponential runtime by Theorem~\ref{thm:lower}, and $\lambda < \log_{\frac{e}{e-1}}n$ will result in exponential runtime by Theorem~\ref{thm:om}~(b).

\section{Beyond the Comma Selection}
\label{sec:beyond}
\subsection{The cGA}
By Theorem~\ref{thm:lower} and Theorem~\ref{thm:om}~$(b)$, we could easily see that although the smaller offspring size is beneficial in escaping from the local optimum, too small offspring size $\lambda \le  (1-\epsilon)\log_{\frac{e}{e-1}}n$ will result in the exponential time in reaching the global optimum of the \om problem, and thus the \omt problem. 
An ideal way is to maintain the good ability of escaping the local optimum as well as an efficient searching ability over the \om problem. Recalling that the cGA utilizes only two samples to update the probabilistic model and it has achieved good performance among many pseudo-Boolean problems~\cite{FriedrichKKS17,Doerr20algo}, we  analyze its performance on the \omt problem. 

As mentioned in Section~\ref{subsec:omt}, this paper only discusses solving the \omt problem in the offline mode with $n$-bit string encoding for the current search point and storing the previous solution for the fitness evaluation. The cGA algorithm for the problem requiring two consecutive time steps is similar to the original cGA, and the only difference is that the winner of the two samples (if two samples have the same fitness, we accept the first one~\cite{LenglerSW20}, and call it the winner) in the last step is stored for the fitness evaluation. In the following, we still call it the cGA. The pseudo-code is shown in Algorithm~\ref{alg:cGA}. The global optimum is that ${X}^{g+1}=1^n$ condition on $X_1^g=0$ for some generation $g\ge0$. 
\begin{algorithm}[!ht]
\caption{The cGA to maximize fitness function $f$ requiring two consecutive time steps}
{\small
 \begin{algorithmic}[1]
 \STATE{$p^0=(\tfrac{1}{2}, \tfrac{1}{2},\dots,\tfrac{1}{2})\in [0,1]^n$, and sample $X^0$ based on $p^0$}
 \FOR {$g=1,2,\dots$}
 \STATE {Independently sample two individuals $X^{(1)g}$ and $X^{(2)g}$ based on $p^{g-1}$}
 \STATEx {$\quad\%\%$\textsl{Update the frequency vector and store the winner between $X^{(1)g}$ and $X^{(2)g}$}}
 \IF{$f(X^{g-1},X^{(1)g}) \ge f(X^{g-1},X^{(2)g})$}
 \STATE {$p'=p^{g-1}+\tfrac{1}{\mu}(X^{(1)g}-X^{(2)g})$};
 \STATE {$X^g=X^{(1)g}$};
 \ELSE 
  \STATE {$p'=p^{g-1}+\tfrac{1}{\mu}(X^{(2)g}-X^{(1)g})$};
   \STATE {$X^g=X^{(2)g}$};
  \ENDIF
 \STATE {$p^g=\min \{\max\{\tfrac{1}{n},p'\},1-\tfrac{1}{n}\}$};
 \ENDFOR
 \end{algorithmic}
 \label{alg:cGA}
}
\end{algorithm}

\subsection{Runtime Analysis}\label{subsec:cga}
Similar to the \oclea discussed in Section~\ref{sec:oclea}, with probability of $1$, the cGA can find the optimum of the \omt problem due to its non-elitism. The search process of the cGA on the \omt problem is identical to the original (non-time-linkage) cGA on the \om problem. But once $(1^n)$, the optimum of the \om problem, is reached, the non-time-linkage theory is not interested in the process afterwards as they only focus on the first hitting time of $(1^n)$ while for the cGA on the \omt problem, the optimum is $(1^n)$ is reached condition on the stored solution with the first bit value of $0$, and there is a good chance that the stored solution has $1$ for the first bit and we need to focus on the process afterwards. The following theorem from~\cite{SudholtW19} shows the runtime for the cGA on the \om problem, which will also be used for our runtime analysis for the \omt problem.
\begin{theorem}[\cite{SudholtW19}]
Consider using the (non-time-linkage) cGA to solve the $n$-dimensional \om problem,
for $\mu\ge c\sqrt n \log n$ with $c>0$ sufficiently large, and $\mu \in \Poly(n)$, the expected runtime is $O(\sqrt n \mu)$.
\label{thm:cgaom}
\end{theorem}

Now we establish our runtime result for the cGA on the time-linkage \omt problem.
\begin{theorem}
Let $n\in\N_{\ge 2}$. Consider using the cGA with population size $\mu=\Omega(\sqrt n\log n) \cap \Poly(n)$ to solve the $n$-dimensional \omt problem. The expected number of fitness evaluations is $O(n^{2.5} \mu)$. In particular, this runtime is $O(n^3\log n)$ when the population size $\mu=\Theta(\sqrt n\log n)$.
\label{thm:cga}
\end{theorem}
\begin{proof}
From the proof in~\cite{SudholtW19}, $O(\sqrt n \mu)$ is the expected runtime for the all dimensions of the frequency $p$ become $1-\frac1n$. 
In the following, we start from this state.

If $X^{g}_1=1$ when all frequency borders $1-\frac 1n$ are reached at a certain generation $g$, in order to reach the optimum of the \omt problem, we require $X^{g+1}_1=0$, that can be easily satisfied by sampling $X^{(1)g+1}=X^{(2)g+1}=(01^{n-1})$, which happens with probability
\begin{align}
\left(\frac{1}{n}\left(1-\frac{1}n\right)^{n-1}\right)^2\ge \frac1{e^{2}n^2}.
\label{eq:0}
\end{align}
Once $X^{g+1}_1=0$, since $X^{(1)g+1}=X^{(2)g+1}=(01^{n-1})$ we know that the frequency $p^{g+1}$ stays at $(1-\frac 1n,\dots,1-\frac 1n)$. Hence, with probability 
\begin{equation}
\begin{split}
1-\left(1-\left(1-\frac{1}{n}\right)^n\right)^2 \ge &1-\left(1-\frac{1}{e} \left(1-\frac{1}{n}\right)\right)^2 
\ge \frac{4e-1}{4e^2},
\end{split}
\label{eq:1n}
\end{equation}
where the last inequality uses $n\ge 2$,
at least one of the two samples becomes $1^n$ and the optimum is reached. Otherwise, we will need to again wait some $g''>g'+1$ that $p^{g''}=(1-\frac 1n,\dots,1-\frac 1n)$, which requires $O(\mu\sqrt n)$ by Theorem~\ref{thm:cgaom}.
By (\ref{eq:0}) and (\ref{eq:1n}), we know that the probability to reach the optimum is at least 
$
\frac1{e^{2}n^2} \frac{4e-1}{4e^2}=\frac{4e-1}{4e^4n^2}.
$
With Wald's equation, we know that the expected generations for the cGA to reach the optimum of the \omt problem is at most $\frac{4e^4}{4e-1}n^2O(\mu\sqrt n)=O(n^{2.5}\mu)$.
\end{proof}
\subsection{Some Notes}
From Sections~\ref{sec:oclea} and~\ref{subsec:cga}, we see the efficiency of the non-elitist algorithms. Although non-elitist algorithms, in which the inferior solution has the chance to enter into the next generation, can ensure the probability of $1$ to reach the optimum of the \omt problem as discussed in Section~\ref{sec:oclea}, it is not true that all non-elitist algorithms can solve \omt as efficiently as the \oclea and the cGA. One example is the Metropolis algorithm~\cite{MetropolisRRTT53}, which is identical to the randomized local search but the inferior solution $y$ generated from $x$ can enter into the next generation with probability $e^{\alpha(f(y)-f(x))}$ where $\alpha > 0$ is a constant. The time-linkage version is similar to the original one but stores the previous solution for the fitness evaluation. We omit the details and just give a note. With probability of around $\frac14$, (0(1*)) local optimum is reached in the $1$st generation. For any generated offspring $\tilde{X}^1$ from $X^1=1*$, its fitness $f(X^1,\tilde{X}^1)\le0$, which with high probability is smaller than $f(X^1,\tilde{X}^1)$ by $\Theta(n)$, and thus has probability of $e^{-\Theta(n)}$ to accept this $\tilde{X}^1$ for possibly escaping the local optimum. Hence, it will require at least $e^{\Theta(n)}$ expected runtime to reach the optimum.

\section{Experiments}
\label{sec:exp}
In this section, we experimentally verify the efficiency of the non-elitist \oclea and cGA discussed in Sections~\ref{sec:oclea} and~\ref{sec:beyond}. As a comparison, we also collect the number of failed runs of the \oplea as Section~\ref{sec:oplea} points out its non-convergence. 

\subsection{Experimental Settings}
We perform 20 independent runs of each of the \oplea, the \oclea and the cGA to see their actual performance. The details of the experimental settings are listed in the following.
\begin{itemize}
\item Problem size $n$: $500,1000,1500,2000,2500,3000$.
\item Offspring size $\lambda$ of the \oclea: $\lceil \log_{\frac{e}{e-1}} n \rceil$ as suggested in Theorem~\ref{thm:oclea}.
\item Offspring size $\lambda$ of the \oplea: $\lceil \log_{\frac{e}{e-1}} n \rceil$, the same value for the \oclea. For the \oplea, if one of the two kinds of local optima is reached, the algorithm terminates.
\item Population size $\mu$ of the cGA: $2\lceil \tfrac14 \sqrt{n} \ln n \rceil$. $\Omega(\sqrt n \ln n)$ is suggested in~\cite{SudholtW19} for the \om problem to prevent the genetic drift effect~\cite{DoerrZ20tec,KrejcaW20}. Since $64$ is the optimal choice of $\mu$ for the $500$-dimensional \om problem in~\cite{DoerrZ20}, we choose $\tfrac 12$ as the proper coefficient to obtain a $\mu$ value close to $64$ for $n=500$, and further use $2\lceil \tfrac14 \sqrt{n} \ln n \rceil$ to make $\mu$ an even number, corresponding to the well-behaved frequency assumption usually made in theory works~\cite{Doerr19tcs}.
\end{itemize}

\subsection{Experimental Results}
\subsubsection{\oplea} For the $20$ independent runs of the \oplea on the \omt with all dimension sizes, $0$ run succeeded and all runs reached one of the two local optima. We note that by Theorem~\ref{thm:oplea}, the upper bound for the probability of a success run is $3n\left(\frac{2e^2}{n^{2/3}\ln \lambda}\right)^{\ln \lambda} +\frac{3\lambda}{n^{1/3}}+ne^{-\frac{n^{1/3}}{e}}$. Even for the item $\frac{3\lambda}{n^{1/3}}$, to require it less than $1$, we need $n$ at least the amount of $700,000$. Hence, it means that the non-convergence can be witnessed even for such medium sizes, which in addition verifies the bad ability of the \oplea to reach the optimum of the \omt problem.

\subsubsection{\oclea and cGA} Figure~\ref{fig:runtime} plots the median number of the fitness evaluations of the \oclea and cGA, together with the first and third quartiles. It is obvious to see that both algorithms can efficiently solve the \omt problem. Besides, the superiority of the cGA for the asymptotic complexity $O(n^3\log n)$ seen in Theorem~\ref{thm:cga} comparing with $O(n^4\log n)$ of the \oclea in Theorem~\ref{thm:oclea} can also be witnessed in the experimental results. 

\begin{figure}[!ht]
\centering
\includegraphics[width=3.9in]{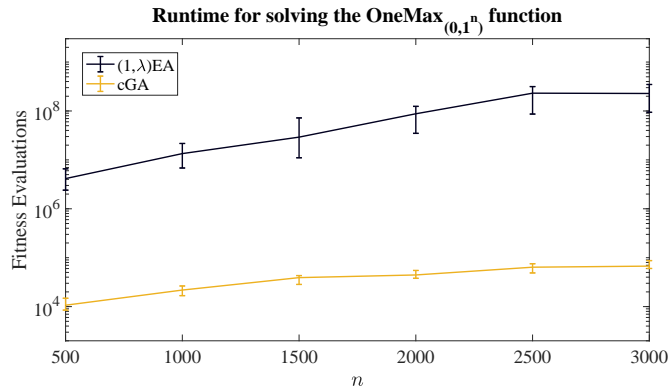}
\caption{The median number of fitness evaluations (with the $1$st and $3$rd quartiles) of the \oclea and the cGA on the \omt problem with problem size $n=500,1000,1500,2000,2500$ and $3000$ in $20$ independent runs.}
\label{fig:runtime}
\end{figure}

\section{Conclusion and Outlook}
\label{sec:con}
Many real-world applications have the time-linkage property, however, there is no theoretical work discussing how the non-elitist algorithms solve the time-linkage problem. The different preference between the current solution and the historical solutions could result in local optima that cause stagnation for some elitist algorithms when we just encode the current solution and store the historical ones for the fitness evaluation. In this paper, based on the recently proposed \omt problem, we proved that the elitist \oplea with a high probability cannot reach the global optimum. In contrast, we proved that non-elitism might help. We proved that the \oclea, the non-elitist counterpart of the \oplea, can reach the optimum in $O(n^4\log n)$ with $\lambda=\log_{\frac{e}{e-1}}n$. Inspired by that the smaller offspring size is helpful to escape from local optima, we resorted to the cGA, and proved its runtime of $O(n^3\log n)$ with population size $\mu=\Theta(\sqrt n \log n)$. Our experiments also verify the poor performance of the \oplea, and the efficiency of the \oclea and the cGA, especially the cGA. Our work provided a positive theoretical evidence for the usage of non-elitism, which is not much in the current evolutionary community. 

There are still many research questions remain open.  There are many future works. The currently analyzed \omt problem only considers one bit position of only the immediate previous time step, and only consider the weight of $-n$ for the previous first bit position. In the future, we will consider more generalized and practical model, like with other weights, more than one bit positions, and more than one historical time steps. The complicated model will result in complicated stochastic dependencies and thus difficult theoretical analysis. We might resort the mean-field analysis~\cite{DoerrZ20tcs} to obtain the approximate theoretical results, and we will consider the more specialized mathematical tools for the time-linkage problems.

\section*{Acknowledgments}
This work was supported by Guangdong Basic and Applied Basic Research Foundation (Grant No. 2019A1515110177), Guangdong Provincial Key Laboratory (Grant No. 2020B121201001), the Program for Guangdong Introducing Innovative and Enterpreneurial Teams (Grant No. 2017ZT07X386), Shenzhen Science and Technology Program (Grant No. KQTD2016112514355531).


\begin{thebibliography}{MRR{\etalchar{+}}53}

\bibitem[Bos05]{Bosman05}
Peter~AN Bosman.
\newblock Learning, anticipation and time-deception in evolutionary online
  dynamic optimization.
\newblock In {\em Genetic and Evolutionary Computation Conference Companion,
  GECCO 2005 Workshop}, pages 39--47. ACM, 2005.

\bibitem[COY18]{CorusOY18}
Dogan Corus, Pietro~S Oliveto, and Donya Yazdani.
\newblock Fast artificial immune systems.
\newblock In {\em International Conference on Parallel Problem Solving from
  Nature, {PPSN 2018}}, pages 67--78. Springer, 2018.

\bibitem[COY20]{CorusOY20}
Dogan Corus, Pietro~S Oliveto, and Donya Yazdani.
\newblock When hypermutations and ageing enable artificial immune systems to
  outperform evolutionary algorithms.
\newblock {\em Theoretical Computer Science}, 832:166--185, 2020.

\bibitem[CTCY10]{ChenTCY10}
Tianshi Chen, Ke~Tang, Guoliang Chen, and Xin Yao.
\newblock Analysis of computational time of simple estimation of distribution
  algorithms.
\newblock {\em IEEE Transactions on Evolutionary Computation}, 14(1):1--22,
  2010.

\bibitem[CTCY12]{ChenTCY12}
Tianshi Chen, Ke~Tang, Guoliang Chen, and Xin Yao.
\newblock A large population size can be unhelpful in evolutionary algorithms.
\newblock {\em Theoretical Computer Science}, 436:54--70, 2012.

\bibitem[DJW02]{DrosteJW02}
Stefan Droste, Thomas Jansen, and Ingo Wegener.
\newblock On the analysis of the (1+1) evolutionary algorithm.
\newblock {\em Theoretical Computer Science}, 276(1-2):51--81, 2002.

\bibitem[DK20]{DoerrK20}
Benjamin Doerr and Martin~S Krejca.
\newblock The univariate marginal distribution algorithm copes well with
  deception and epistasis.
\newblock In {\em European Conference on Evolutionary Computation in
  Combinatorial Optimisation, {EvoCOP 2020}}, pages 51--66. Springer, 2020.

\bibitem[Doe19]{Doerr19tcs}
Benjamin Doerr.
\newblock Analyzing randomized search heuristics via stochastic domination.
\newblock {\em Theoretical Computer Science}, 773:115--137, 2019.

\bibitem[Doe20a]{Doerr20}
Benjamin Doerr.
\newblock Does comma selection help to cope with local optima?
\newblock In {\em Genetic and Evolutionary Computation Conference, {GECCO
  2020}}, pages 1304--1313. ACM, 2020.

\bibitem[Doe20b]{Doerr20algo}
Benjamin Doerr.
\newblock The runtime of the compact genetic algorithm on jump functions.
\newblock {\em Algorithmica}, pages 1--49, 2020.

\bibitem[Dro06]{Droste06}
Stefan Droste.
\newblock A rigorous analysis of the compact genetic algorithm for linear
  functions.
\newblock {\em Natural Computing}, 5(3):257--283, 2006.

\bibitem[DZ20a]{DoerrZ20}
Benjamin Doerr and Weijie Zheng.
\newblock From understanding genetic drift to a smart-restart parameter-less
  compact genetic algorithm.
\newblock In {\em Genetic and Evolutionary Computation Conference, {GECCO
  2020}}, pages 805--813. ACM, 2020.

\bibitem[DZ20b]{DoerrZ20tec}
Benjamin Doerr and Weijie Zheng.
\newblock Sharp bounds for genetic drift in estimation of distribution
  algorithms.
\newblock {\em IEEE Transactions on Evolutionary Computation},
  24(6):1140--1149, 2020.

\bibitem[DZ20c]{DoerrZ20tcs}
Benjamin Doerr and Weijie Zheng.
\newblock Working principles of binary differential evolution.
\newblock {\em Theoretical Computer Science}, 801:110--142, 2020.

\bibitem[FKKS17]{FriedrichKKS17}
Tobias Friedrich, Timo K{\"o}tzing, Martin~S Krejca, and Andrew~M Sutton.
\newblock The compact genetic algorithm is efficient under extreme gaussian
  noise.
\newblock {\em IEEE Transactions on Evolutionary Computation}, 21(3):477--490,
  2017.

\bibitem[GK16]{GiessenK16}
Christian Gie{\ss}en and Timo K{\"o}tzing.
\newblock Robustness of populations in stochastic environments.
\newblock {\em Algorithmica}, 75(3):462--489, 2016.

\bibitem[GKS99]{GarnierKS99}
Josselin Garnier, Leila Kallel, and Marc Schoenauer.
\newblock Rigorous hitting times for binary mutations.
\newblock {\em Evolutionary Computation}, 7(2):173--203, 1999.

\bibitem[HS18]{HasenohrlS18}
V{\'a}clav Hasen{\"o}hrl and Andrew~M Sutton.
\newblock On the runtime dynamics of the compact genetic algorithm on jump
  functions.
\newblock In {\em Genetic and Evolutionary Computation Conference, {GECCO
  2018}}, pages 967--974. ACM, 2018.

\bibitem[JS07]{JagerskupperS07}
Jens Jagerskupper and Tobias Storch.
\newblock When the plus strategy outperforms the comma strategyand when not.
\newblock In {\em IEEE Symposium on Foundations of Computational Intelligence,
  {FOCI 2007}}, pages 25--32. IEEE, 2007.

\bibitem[JW07]{JansenW07}
Thomas Jansen and Ingo Wegener.
\newblock A comparison of simulated annealing with a simple evolutionary
  algorithm on pseudo-boolean functions of unitation.
\newblock {\em Theoretical Computer Science}, 386(1-2):73--93, 2007.

\bibitem[KW20]{KrejcaW20}
Martin~S Krejca and Carsten Witt.
\newblock Theory of estimation-of-distribution algorithms.
\newblock In Benjamin Doerr and Frank Neumann, editors, {\em Theory of
  Evolutionary Computation}, pages 405--442. Springer, 2020.

\bibitem[LN19]{LehreN19}
Per~Kristian Lehre and Phan Trung~Hai Nguyen.
\newblock On the limitations of the univariate marginal distribution algorithm
  to deception and where bivariate edas might help.
\newblock In {\em ACM/SIGEVO Conference on Foundations of Genetic Algorithms,
  {FOGA 2019}}, pages 154--168. ACM, 2019.

\bibitem[LOW19]{LissovoiOW19}
Andrei Lissovoi, Pietro~S Oliveto, and John~Alasdair Warwicker.
\newblock On the time complexity of algorithm selection hyper-heuristics for
  multimodal optimisation.
\newblock In {\em AAAI Conference on Artificial Intelligence, {AAAI 2019}},
  volume~33, pages 2322--2329, 2019.

\bibitem[LSW20]{LenglerSW20}
Johannes Lengler, Dirk Sudholt, and Carsten Witt.
\newblock The complex parameter landscape of the compact genetic algorithm.
\newblock {\em Algorithmica}, pages 1--42, 2020.

\bibitem[MRR{\etalchar{+}}53]{MetropolisRRTT53}
Nicholas Metropolis, Arianna~W. Rosenbluth, Marshall~N. Rosenbluth, Augusta~H.
  Teller, and Edward Teller.
\newblock Equation of state calculations by fast computing machines.
\newblock {\em The Journal of Chemical Physics}, 21:1087--1092, 1953.

\bibitem[Ngu11]{Nguyen11}
Trung~Thanh Nguyen.
\newblock {\em Continuous dynamic optimisation using evolutionary algorithms}.
\newblock PhD thesis, University of Birmingham, 2011.

\bibitem[OPH{\etalchar{+}}18]{OlivetoPHST18}
Pietro~S Oliveto, Tiago Paix{\~a}o, Jorge~P{\'e}rez Heredia, Dirk Sudholt, and
  Barbora Trubenov{\'a}.
\newblock How to escape local optima in black box optimisation: When
  non-elitism outperforms elitism.
\newblock {\em Algorithmica}, 80(5):1604--1633, 2018.

\bibitem[PHST17]{PaixaoHST17}
Tiago Paix{\~a}o, Jorge~P{\'e}rez Heredia, Dirk Sudholt, and Barbora
  Trubenov{\'a}.
\newblock Towards a runtime comparison of natural and artificial evolution.
\newblock {\em Algorithmica}, 78(2):681--713, 2017.

\bibitem[RS14]{RoweS14}
Jonathan~E Rowe and Dirk Sudholt.
\newblock The choice of the offspring population size in the (1, $\lambda$)
  evolutionary algorithm.
\newblock {\em Theoretical Computer Science}, 545:20--38, 2014.

\bibitem[Sud20]{Sudholt20}
Dirk Sudholt.
\newblock The benefits of population diversity in evolutionary algorithms: a
  survey of rigorous runtime analyses.
\newblock {\em Theory of Evolutionary Computation}, pages 359--404, 2020.

\bibitem[SW19]{SudholtW19}
Dirk Sudholt and Carsten Witt.
\newblock On the choice of the update strength in estimation-of-distribution
  algorithms and ant colony optimization.
\newblock {\em Algorithmica}, 81(4):1450--1489, 2019.

\bibitem[ZCY21]{ZhengCY21}
Weijie Zheng, Huanhuan Chen, and Xin Yao.
\newblock Analysis of evolutionary algorithms on fitness function with
  time-linkage property.
\newblock {\em IEEE Transactions on Evolutionary Computation}, 2021.
\newblock In Press.

\end{thebibliography}


}
\end{document}